\documentclass[journal]{IEEEtran}
\usepackage{amssymb}
\usepackage{amsmath}
\usepackage{dcolumn}
\usepackage{bm}
\usepackage{color}
\usepackage{graphicx}
\usepackage{subfigure} 
\usepackage{tabularx}
\usepackage{array}
\usepackage{amsthm}
\usepackage{cite}
\usepackage[ruled,linesnumbered]{algorithm2e}
\usepackage{tabu}                     
\usepackage{multirow}                 
\usepackage{multicol}                 
\usepackage{multirow}                
\usepackage{float}                    
\usepackage{makecell}                 
\usepackage{booktabs}
\usepackage{threeparttable}  

\usepackage{etoolbox}
\makeatletter
\patchcmd{\@makecaption}
  {\scshape}
  {}
  {}
  {}
\makeatother

\SetKwInput{KwData}{Input}
\SetKwInput{KwResult}{Output}

\usepackage[colorlinks,
            linkcolor=blue,
            anchorcolor=green,
            citecolor=red
            ]{hyperref}

\newtheorem{theorem}{Theorem}
\newtheorem{corollary}{Corollary}
\newtheorem{lemma}{Lemma}
\newtheorem{remark}{Remark}
\newtheorem{definition}{Definition}
\newtheorem{assumption}{Assumption}
\newtheorem{proposition}{Proposition}
\newtheorem{problem}{Problem}

\renewcommand{\t}{^{\mbox{\tiny\sf T}}}
\newcommand{\bremark}{\begin{remark}
\begin{rm}}
\newcommand{\eremark}{ \end{rm}\hfill \rule{1mm}{2mm}
\end{remark} }
\newcommand{\btheorem}{\begin{theorem} \begin{it}}
\newcommand{\etheorem}{\end{it} \hfill \rule{1mm}{2mm}
\end{theorem} }
\newcommand{\blemma}{\begin{lemma} \begin{it} }
\newcommand{\elemma}{ \end{it} \hfill\rule{1mm}{2mm}
\end{lemma} }
\newcommand{\bcorollary}{\begin{corollary} \begin{it} }
\newcommand{\ecorollary}{ \end{it} \hfill\rule{1mm}{2mm}
\end{corollary} }
\newcommand{\bdefinition}{\begin{definition} }
\newcommand{\edefinition}{ \hfill\rule{1mm}{2mm}
\end{definition} }
\newcommand{\bproposition}{\begin{proposition} }
\newcommand{\eproposition}{\hfill \rule{1mm}{2mm}
\end{proposition} }
\newcommand{\bexample}{\begin{example} \begin{rm}}
\newcommand{\eexample}{ \end{rm} \hfill\rule{1mm}{2mm}
\end{example} }

\newcommand{\bluff}{{\hbox{\raise 15pt \hbox{\hskip 0.5pt}}}}

\newcommand{\bassumption}{\begin{assumption} }
\newcommand{\eassumption}{\hfill \rule{1mm}{2mm}
\end{assumption} }

\newcommand{\balgorithm}{\medskip\begin{algorithm} \rm}
\newcommand{\ealgorithm}{ \hfill \rule{1mm}{2mm}\medskip
\end{algorithm} }

\newcommand{\basm}{\begin{assumption} \begin{rm} }
\newcommand{\easm}{ \end{rm} \hfill\rule{1mm}{2mm}
\end{assumption} }

\begin{document}

\title{ Equilibrium-Driven Smooth Separation and Navigation of Marsupial Robotic Systems}

\author{
\IEEEauthorblockN{Bin-Bin Hu, Bayu Jayawardhana, and Ming Cao
}
\thanks{
This work was supported the NXTGEN Hightech program
on Autonomous Factory (https://nxtgenhightech.nl/).
}
\thanks{
Bin-Bin Hu, Bayu Jayawardhana, and Ming Cao are with the Engineering and Technology Institute Groningen, Faculty of Science and Engineering, University of Groningen, 9747 AG Groningen, The Netherlands (e-mail: b.hu@rug.nl; b.jayawardhana@rug.nl, m.cao@rug.nl).
}
}

\maketitle

\begin{abstract}
In this paper, we propose an equilibrium-driven controller that enables a marsupial carrier-passenger robotic system to achieve smooth carrier-passenger separation and then to navigate the passenger robot toward a predetermined target point. Particularly, we design a potential gradient in the form of a cubic polynomial for the passenger's controller as a function of the carrier-passenger and carrier-target distances in the moving carrier's frame. This introduces multiple equilibrium points corresponding to the zero state of the error dynamic system during carrier-passenger separation. The change of equilibrium points is associated with the change in their attraction regions, enabling smooth carrier–passenger separation and afterwards seamless navigation toward the target. Finally, simulations demonstrate the effectiveness and adaptability of the proposed controller in environments containing obstacles.
\end{abstract}


\IEEEpeerreviewmaketitle
\section{Introduction}
Over the years, the coordination of multi-robot systems (MRSs) has played an important role in a wide range of applications, including urban search and rescue, environmental monitoring, delivery, and collaborative convoying \cite{camisa2022multi,hu2023spontaneous,hu2024coordinated}. Among these various MRSs, the heterogeneous marsupial robotic systems, each comprising a large ``carrier" robot that transports a small ``passenger" robot, have garnered significant attention. Such a unique architecture can extend mission complexity and effectively provide broader spatial coverage by combining the strengths of carrier and passenger robots \cite{murphy1999marsupial,murphy2000marsupial}. For instance, in environments with hazardous targets, vertical shafts, and dense obstacles that cannot be reached by the carrier robot, such as an unmanned ground vehicle (UGV) or unmanned surface vehicle (USV), the addition of an unmanned aerial vehicle (UAV) acting as a passenger robot can overcome these challenges at a low cost.

In the initial stage, most of the efforts have been devoted to the architectural design of various marsupial systems \cite{janssen2007enabling}. For the USV-UAV marsupial systems, an autonomous surface-aerial robotic team was developed in \cite{pinto2014autonomous} for riverine environmental monitoring. Subsequently, an advanced platform capable of generating detailed maps and operating in remote locations was proposed in \cite{kalaitzakis2021marsupial} to support freshwater ecosystem studies. Additionally, some works have explored USV-UAV coordinated landing missions \cite{zhang2021visual}. A bio-inspired marsupial fish system, consisting of heterogeneous robotic fishes, was introduced in \cite{zhou2010marsupial}. For marsupial systems involving a UAV and a locomotion dog robot, a pipeline utilizing deep neural network-based vehicle-to-vehicle detection was developed in \cite{arora2023deep}. For UGV-UAV marsupial systems, the design of hardware, software architecture, and detailed experiments of a tethered UAV was presented in \cite{martinez2024design}, which significantly extends operational endurance.

Besides the specialized design of marsupial robotic platforms, another critical challenge lies in how to plan the deployment or separation of passenger robots \cite{moore2016nested}, the failure of which can prevent the passenger robot from contributing effectively to the mission. In this endeavor, early works \cite{las2015path,stankiewicz2018motion} relied on predefined Boolean conditions, such as determining whether a target point is accessible by the passenger robot. While relatively straightforward, these methods \cite{las2015path,stankiewicz2018motion} cannot handle complex scenarios. To address these limitations, a speed-management approach was proposed in \cite{mei2006deployment} to maximize the travel distance while accounting for energy and timing constraints. A temporal symbolic planning method was introduced in \cite{wurm2013coordinating} to assign targets for deploying or separating smaller robots effectively. For more complex long-term deployment tasks, a sequential stochastic assignment approach leveraging prior probability distributions was proposed in \cite{lee2021optimal}. Building on this work, a Monte-Carlo-Tree-Search framework was developed in \cite{lee2021stochastic}, enabling the carrier robot to optimally plan deployment or separation times and locations while reasoning over uncertain future observations and rewards. However, the aforementioned works \cite{moore2016nested,las2015path,stankiewicz2018motion,mei2006deployment,wurm2013coordinating,lee2021optimal,lee2021stochastic} have not considered sufficiently the safety when the passenger and carrier robots separate via event-based techniques. Specifically, since the carrier and passenger robots remain static before separation, a sudden large input change for the passenger robot at the start stage of deployment may cause the potential robot damage and deployment failure.


To this end, we propose an equilibrium-driven controller that enables a marsupial carrier-passenger system to achieve the smooth carrier-passenger separation and then to navigate the passenger robot to a predetermined target point. Particularly, we design a potential gradient in the form of a cubic polynomial for the passenger's controller as a function of the carrier-passenger and carrier-target distances in the moving carrier's frame, which introduces three equilibrium points in the carrier-passenger error dynamics. Each corresponding to one of the following three scenarios: the passenger robot is (i) staying on the carrier robot, (ii) positioned between the carrier robot and the target, and (iii) reaching the target. Initially, the passenger robot coincides with the carrier robot, keeping it stationary within the attraction region between equilibrium points (i) and (ii). By appropriately decreasing the carrier-target distance, the equilibrium point (ii) moves closer to and eventually coincides with the equilibrium point (i), thereby decreasing the attraction region of equilibrium point (i) to zero and increasing that of equilibrium point (iii) accordingly. This transition guides the passenger robot to leave the equilibrium point (i) and approach the equilibrium point (iii), which conveniently determines the passenger robot's stay on or smooth separation from the carrier robot. The main contribution is summarized as follows.

 To the best of our knowledge, it is the first time that an equilibrium-driven strategy is proposed for achieving smooth separation and navigation in heterogeneous marsupial robotic systems. Unlike traditional event-based separation techniques, the proposed equilibrium-driven strategy ensures that as the passenger robot transitions from one equilibrium to another, the closer it is to the equilibrium, the smaller the required control input. This advantage enables passenger's smooth separation from the carrier and convergence to the target, avoiding sudden changes/discontinuity in the applied control inputs. Furthermore, this strategy can be integrated with existing planning and control methods to enhance safety, demonstrating potential for various applications such as urban delivery and search and rescue.

The rest of the paper is organized below. Section~\ref{sec_preliminaries} introduces some 
preliminaries and formulates the problem. Section~\ref{sec_controller_design} details the controller for the carrier and passenger robots. Section~\ref{sec_convergence} presents the convergence analysis. Section~\ref{sec_pimulation} conducts 2D and 3D simulations to verify the algorithm's effectiveness. Finally, Section~\ref{sec_conclusion}
draw the conclusion.


\section{Preliminaries}
\label{sec_preliminaries}

We consider a marsupial robotic system $\mathcal V=\{ \nu_c, \nu_p\}$, which comprises a large carrier robot $\nu_c$ that transports a small passenger robot $\nu_p$ \cite{stankiewicz2018motion}. Here, the carrier and passenger robots are moving in the $n$-dimensional Euclidean space with the single integrator dynamics,
\begin{align}
\label{marsupial_dynamics}
\dot{\bold{x}}_\textrm{c}=\bold{u}_\textrm{c}, \dot{\bold{x}}_\textrm{p}=\bold{u}_\textrm{p}, 
\end{align}
where $\bold{x}_\textrm{c}:=[x_\textrm{c,1}, \cdots, x_\textrm{c,n}]\t\in\mathbb{R}^n, \bold{x}_\textrm{p}:=[x_\textrm{p,1}\in\mathbb{R}^n, \cdots, x_\textrm{p,n}]\t\in\mathbb{R}^n, \bold{u}_\textrm{c}:$ $=[u_\textrm{c,1}, \cdots, u_\textrm{c,n}]\t\in\mathbb{R}^n, \bold{u}_\textrm{p}:=[u_\textrm{p,1}, \cdots, u_\textrm{p,n}]\t\in\mathbb{R}^n$ represent the positions and inputs of the carrier and passenger robots, respectively. If it is more desirable to work with higher-order complicated dynamics such as UGVs, UAVs, and USVs\cite{aggarwal2020path,hu2024ordering}, one can treat the inputs $\bold{u}_\textrm{c}, \bold{u}_\textrm{p}$ as the upper-level desired velocities with an additional lower-level velocity tracking module.



We also consider a target point $\bold{x}_\textrm{t}:=[x_\textrm{t,1}, \cdots, x_\textrm{t,n}]\t\in\mathbb{R}^n$ to be fixed in time. In our problem setting, we emphasize that this target point in our problem setting is
hazardous and unable to reach by the carrier robot, which implies it can only be explored by the small passenger robot in \eqref{marsupial_dynamics}. 






\begin{definition}
\label{def_peparation_navigation}
 (Marsupial Separation \& Navigation) The carrier-passenger marsupial system $\mathcal V$ governed by \eqref{marsupial_dynamics} collectively achieves 
 the marsupial separation and navigation to the fixed target point if the following three properties are satisfied,
 
 \begin{itemize}
 
\item  {\bf P1 (Carrier-passenger separation):}  \label{P_1} The passenger robot starts on the carrier robot and then will achieve the carrier-passenger separation after a finite time $T\in\mathbb{R}^+$, i.e., $~\|\bold{x}_{c}(t)-\bold{x}_\textrm{p}(t)\|=0,~\forall t\in[0, T]$, and $\|\bold{x}_{c}(T^+)-\bold{x}_\textrm{p}(T^+)\|\neq0$, with $\bold{x}_{c}, \bold{x}_{p}$ given in \eqref{marsupial_dynamics}, and $T^+$ being right limit of the time $T$.

\item  {\bf P2 (Passenger-target navigation):}  \label{P_2}  The passenger robot  
asymptotically converges to the target point after the carrier-passenger separation, i.e., $\lim_{t\rightarrow\infty}\bold{x}_\textrm{p}(t)=\bold{x}_\textrm{t}$ with $\bold{x}_\textrm{t}$ being the position of the fixed target point.

\item  {\bf P3 (Carrier-target avoidance):}  \label{P_3} The carrier robot guarantees avoiding the target point all along, i.e., $\|\bold{x}_{c}(t)-\bold{x}_\textrm{t}(t)\|\neq 0, \forall t>0$.

\end{itemize}

\end{definition}




\begin{problem}
\label{problem_formulation}
Design the inputs $\{\bold{u}_\textrm{c}, \bold{u}_\textrm{p}\}$ for the carrier-passenger systems governed by \eqref{marsupial_dynamics} such that the marsupial separation \& navigation in Definition~\ref{def_peparation_navigation} is achieved.
\end{problem}


\begin{assumption}
\label{initial_carrier_target_position}
The initial relative position between the carrier robot and the target point is assumed to be sufficiently large, i.e., $\|\bold{x}_\textrm{c}(0)-\bold{x}_\textrm{t}(0)\|\geq \eta$ for some $\eta>0$. 
\end{assumption}


\begin{assumption}
\label{initial_carrier_passenger_position}
The initial positions of the carrier and passenger robots are the same, i.e., $\bold{x}_\textrm{p}(0)=\bold{x}_\textrm{c}(0)$.
\end{assumption}


\section{Controller Design and Analysis}
\label{sec_controller_design}

Let $\bold{e}_\textrm{p,c}$ be the position error between the passenger and carrier robots
\begin{align}
\label{err_passenger_carrier}
\bold{e}_\textrm{p,c}=\bold{x}_\textrm{p}-\bold{x}_{c},
\end{align}
where $\bold{x}_\textrm{p}, \bold{x}_{c}$ are given in \eqref{marsupial_dynamics}. Analogously, let $\bold{e}_\textrm{p,t}, \bold{e}_\textrm{t,c}$ be the position errors for the passenger- and carrier-target groups, respectively, and given by  
\begin{align}
\label{err_passenger_carrier_carrier}
\bold{e}_\textrm{p,t}=&\bold{x}_\textrm{p}-\bold{x}_\textrm{t},~\bold{e}_\textrm{t,c}=\bold{x}_\textrm{t}-\bold{x}_\textrm{c}.
\end{align}
Using 
\eqref{err_passenger_carrier}, \eqref{err_passenger_carrier_carrier} and Definition~\ref{def_peparation_navigation}, we propose 
the following controllers for 
the carrier and passenger robots 
\begin{equation}
\label{controller_carrier}
\bold{u}_\textrm{c}=\left\{
\begin{aligned}
k_\textrm{c}\bold{e}_\textrm{t,c}, & &\mathrm{if}~\|\bold{e}_\textrm{p,c}\|=0,\\
\bold{0}_\textrm{n}, & &\mathrm{if}~ \|\bold{e}_\textrm{p,c}\|\neq0,
\end{aligned}
\right.
\end{equation}
and
\begin{equation}
\label{controller_passengerr}
\bold{u}_\textrm{p}=\left\{
\begin{aligned}
-P(\bold{e}_\textrm{p,c}, \bold{e}_\textrm{t,c})\bold{e}_\textrm{p,c}+\bold{u}_\textrm{c}, & &\mathrm{if}~P(\bold{e}_\textrm{p,c}, \bold{e}_\textrm{t,c})\geq0,\\
P(\bold{e}_\textrm{p,c}, \bold{e}_\textrm{t,c})\bold{e}_\textrm{p,t}+\bold{u}_\textrm{c}, & &\mathrm{if}~ P(\bold{e}_\textrm{p,c}, \bold{e}_\textrm{t,c})<0,
\end{aligned}
\right.
\end{equation}
where the polynomial potential gradient below
\begin{align}
\label{Poteltion_field}
P(\bold{e}_\textrm{p,c}, \bold{e}_\textrm{t,c})=&k_\textrm{p}(\|\bold{e}_\textrm{p,c}\|-\|\bold{e}_\textrm{t,c}\|)(\|\bold{e}_\textrm{p,c}\|\nonumber\\
&+d)\left(\|\bold{e}_\textrm{p,c}\|-\frac{\|\bold{e}_\textrm{t,c}\|}{b}+c\right).
\end{align}
The parameters $b \in \mathbb{R}^+$, $c \in \mathbb{R}^+, d \in \mathbb{R}^+$ are chosen to satisfy
\begin{align}
\label{bc_condition}
b>1, bc< \eta,
\end{align}
where $\eta$ is the distance lower-bound  between the carrier robot and the target point as in Assumption 1.
Note that the controllers \eqref{controller_carrier} and \eqref{controller_passengerr} are one possible design choice motivated by Definition~\ref{def_peparation_navigation}, while alternative formulations are also possible. The parameters $k_\textrm{c}\in\mathbb{R}^+, k_\textrm{p}\in\mathbb{R}^+$ in \eqref{controller_carrier} and \eqref{Poteltion_field} are the control gains,  
$P(\bold{e}_\textrm{p,c}, \bold{e}_\textrm{t,c})$ can be regarded as a cubic polynomial with respect to $\|\bold{e}_\textrm{p,c}\|$, and $\|\bold{e}_\textrm{p,c}\|, \|\bold{e}_\textrm{t,c}\|$ represent the passenger–carrier and carrier–target distances, respectively.
The necessity of $P(\bold{e}_\textrm{p,c}, \bold{e}_\textrm{t,c})$ is to introduce multiple 
equilibrium points to the carrier-passenger error dynamic system, which can achieve the smooth carrier-passenger separation via the adjustment of different equilibrium points. The parameters $b, c, d$ in \eqref{Poteltion_field} allow adjustment of the zero points of $P(\bold{e}_\textrm{p,c}, \bold{e}_\textrm{t,c}) = 0$ and the function shape. Precisely, for $P(\bold{e}_\textrm{p,c}, \bold{e}_\textrm{t,c})=0$, the parameter $b>1$ in \eqref{bc_condition} ensures that the zero point $\|\bold{e}_\textrm{p,c}\|={\|\bold{e}_\textrm{t,c}\|}/{b}-c$ is always smaller than the zero point $\|\bold{e}_\textrm{p,c}\|=\|\bold{e}_\textrm{t,c}\|$ (i.e., ${\|\bold{e}_\textrm{t,c}\|}/{b}-c<\|\bold{e}_\textrm{t,c}\|$), while the parameter $c$ satisfying $bc< \eta$ in \eqref{bc_condition} further ensures that the initial value of the zero point $\|\bold{e}_\textrm{p,c}\|={\|\bold{e}_\textrm{t,c}(0)\|}/{b}-c$
 is strictly positive, i.e., ${\|\bold{e}_\textrm{t,c}(0)\|}/{b}-c> {\eta}/{b}-c>0$ under Assumption~\ref{initial_carrier_target_position}. As shown in Fig.~\ref{potential_field} (a), this effectively divides the initial carrier–target region $[0, \|\bold{e}_\textrm{t,c}(0)\|]$ into two intervals according to the sign of $P(\bold{e}_\textrm{p,c}, \bold{e}_\textrm{t,c}(0))$,  namely,
 $P(\bold{e}_\textrm{p,c}, \bold{e}_\textrm{t,c}(0))\geq0, \forall \|\bold{e}_\textrm{p,c}\|\in [0, {\|\bold{e}_\textrm{t,c}\|}/{b}-c]$, and $ P(\bold{e}_\textrm{p,c}, \bold{e}_\textrm{t,c}(0))<0, \forall\|\bold{e}_\textrm{p,c}\|\in ({\|\bold{e}_\textrm{t,c}\|}/{b}-c, \|\bold{e}_\textrm{t,c}(0)\|)$,
which defines the attraction regions for different equilibrium points. Finally, the parameter $d$ adjusts the shape of $P(\bold{e}_\textrm{p,c}, \bold{e}_\textrm{t,c})$, influencing the separation and navigation speed.

For the carrier's controller $\bold{u}_\textrm{c}$ in \eqref{controller_carrier}, taking the derivative of the carrier-target error $\bold{e}_\textrm{t,c}$ in \eqref{err_passenger_carrier_carrier} and substituting \eqref{marsupial_dynamics}, \eqref{controller_carrier} yields
\begin{equation}
\label{err_dynamic_carrier_target}
\dot{\bold{e}}_\textrm{t,c}=\left\{
\begin{aligned}
-k_\textrm{c}\bold{e}_\textrm{t,c}, & &\mathrm{if}~\|\bold{e}_\textrm{p,c}\|=0,\\
\bold{0}_\textrm{n}, & &\mathrm{if}~ \|\bold{e}_\textrm{p,c}\|\neq0,
\end{aligned}
\right.
\end{equation}
which implies that the carrier robot can move close to the fixed target point or stop according to the value of $\|\bold{e}_\textrm{p,c}\|$. Note that the carrier controller~\eqref{controller_carrier} can be replaced by other planning algorithms, such as guiding vector fields for following a desired elliptical path around the target \cite{hu2023spontaneous}. 

For the passenger's controller $\bold{u}_\textrm{p}$ in \eqref{controller_passengerr}, the additional input $\bold{u}_\textrm{c}$ in \eqref{controller_passengerr} is to change the passenger-carrier error $\bold{e}_\textrm{p,c}$ in \eqref{err_passenger_carrier} to the carrier's coordinate. Taking the derivative of $\bold{e}_\textrm{p,c}$ and substituting \eqref{marsupial_dynamics}, \eqref{controller_passengerr} yields
\begin{equation}
\label{err_dynamic_passengerr_carrier}
\dot{\bold{e}}_\textrm{p,c}=\left\{
\begin{aligned}
-P(\bold{e}_\textrm{p,c}, \bold{e}_\textrm{t,c})\bold{e}_\textrm{p,c}, & &\mathrm{if}~P(\bold{e}_\textrm{p,c}, \bold{e}_\textrm{t,c})\geq0,\\
P(\bold{e}_\textrm{p,c}, \bold{e}_\textrm{t,c})\bold{e}_\textrm{p,t}, & &\mathrm{if}~~P(\bold{e}_\textrm{p,c}, \bold{e}_\textrm{t,c})<0,
\end{aligned}
\right.
\end{equation}
which implies that the evolution of $\bold{e}_\textrm{p,c}$ also corresponds to two scenarios according to the value of $P(\bold{e}_\textrm{p,c}, \bold{e}_\textrm{t,c})$ in~\eqref{Poteltion_field}.  Although the passenger controller \eqref{controller_passengerr} switches at $P(\bold{e}_\textrm{p,c}, \bold{e}_\textrm{t,c})=0$, the relative carrier-passenger input $\bold{u}_\textrm{p}-\bold{u}_\textrm{c}$ for the carrier-passenger separation remains smooth because the left and right upper Dini derivatives at $P(\bold{e}_\textrm{p,c}, \bold{e}_\textrm{t,c})=0$ remains zero. See Lemma~\ref{lemma_passenger_caintaining} for more details.


\begin{figure}[!htb]
\centering
\includegraphics[width=7.5cm]{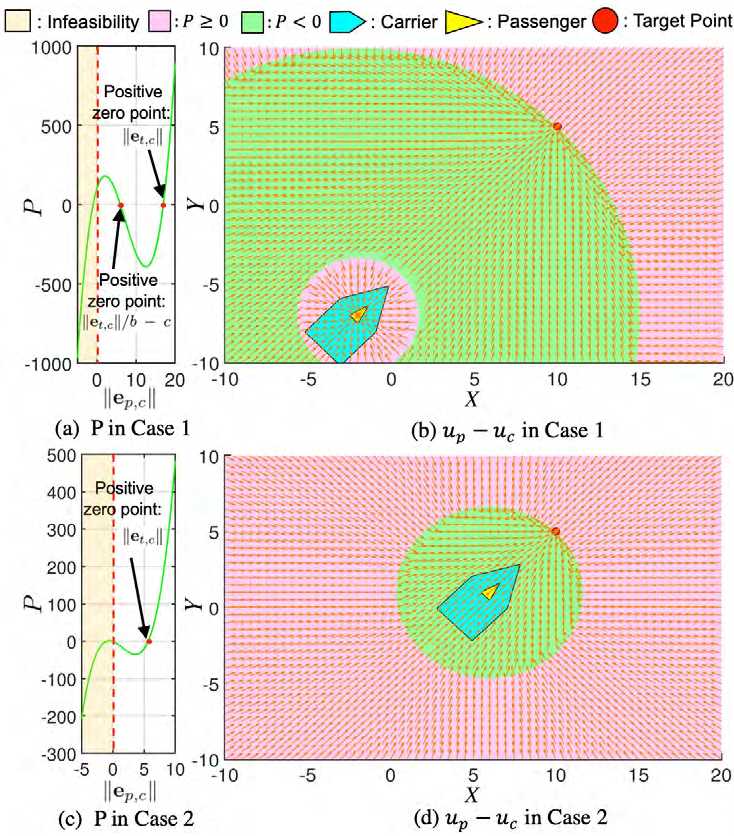}
\caption{Two cases of the carrier-passenger relative inputs $\bold{u}_\textrm{p} - \bold{u}_\textrm{c}$ in \eqref{controller_carrier} and \eqref{controller_passengerr}, under different potential gradient $P(\bold{e}_\textrm{p,c}, \bold{e}_\textrm{t,c})$ in \eqref{Poteltion_field}. Case 1: when the carrier robot is far away from the target point, the potential gradient $P(\bold{e}_\textrm{p,c}, \bold{e}_\textrm{t,c})$ in \eqref{Poteltion_field} contains two positive zero points in Fig.~\ref{potential_field} (a), and the passenger robot remains on the carrier robot under the relative input $\bold{u}_\textrm{p}-\bold{u}_\textrm{c}:=- P(\bold{e}_\textrm{p,c}, \bold{e}_\textrm{t,c})\bold{e}_\textrm{p,c}$ in Fig.~\ref{potential_field}~(b).  Case 2: when the carrier robot is close to the target point, the potential gradient $P(\bold{e}_\textrm{p,c}, \bold{e}_\textrm{t,c})$ in \eqref{Poteltion_field} contains one positive zero point in Fig.~\ref{potential_field} (c), and the passenger robot will separate from the carrier robot under the relative input $\bold{u}_\textrm{p}-\bold{u}_\textrm{c}:= P(\bold{e}_\textrm{p,c}, \bold{e}_\textrm{t,c})\bold{e}_\textrm{p,t}$ in Fig.~\ref{potential_field} (d). Here, the parameters $b, c, d$ in \eqref{Poteltion_field} are set to be $b=3, c=2, d=1$. }
\label{potential_field}
\end{figure}

To clarify the smooth carrier-passenger separation, we illustrate the relative input $\bold{u}_\textrm{p} - \bold{u}_\textrm{c}$ in Fig.~\ref{potential_field}.
As shown in Fig.~\ref{potential_field} (c), the carrier robot is getting close to the target point, and the carrier-target distance becomes $\|\bold{e}_\textrm{t,c}\| = bc$.
Therefore, the zero point of $P(\bold{e}_\textrm{p,c}, \bold{e}_\textrm{t,c})$ is given by $\|\bold{e}_\textrm{p,c}\| = {\|\bold{e}_\textrm{t,c}\|}/{b} - c =0$. 
Under this condition, one has that $P (\bold{e}_\textrm{p,c}, \bold{e}_\textrm{t,c})\rightarrow 0$ as $\|\bold{e}_\textrm{p,c}\| \rightarrow 0$ when $\|\bold{e}_\textrm{t,c}\| = bc$.
This results in small relative input $\bold{u}_\textrm{p} - \bold{u}_\textrm{c}$ in the neighborhood of zero.

\begin{remark}
\label{remark_1}
The reasons that the carrier robot stops after the carrier-passenger separation in \eqref{controller_carrier} are given below. (i)~When the carrier stops (i.e., $\|\bold{e}_\textrm{t,c}(t)\|$ is constant), $P(\bold{e}_\textrm{p,c}, \bold{e}_\textrm{t,c})$ in \eqref{Poteltion_field} depends solely on the relative carrier-passenger distance $\|\bold{e}_\textrm{p,c}\|$, which implies that the coordinate frame of $P(\bold{e}_\textrm{p,c}, \bold{e}_\textrm{t,c})$ can be dynamically established on the carrier's local frame. Such a dynamic frame facilitates smooth carrier-passenger separation and navigation by regulating the attraction regions of distinct equilibrium points. 
However, when implementing the algorithm in practice, an alternative approach for allowing the carrier robot to continue maneuvering is to record and fix $\|\bold{e}_\textrm{t,c}(t)\|=\bold{e}_\textrm{t,c}(T)$ for all $t>T$ at the moment of carrier-target separation because $\|\bold{e}_\textrm{t,c}\|$ becomes constant after the separation, thereby preserving the shape of the potential gradient $P(\bold{e}_\textrm{p,c}, \bold{e}_\textrm{t,c})$.
\end{remark}


\section{Convergence Analysis}
\label{sec_convergence}
In this section, we will present the carrier’s convergence toward the target before the passenger separation in Lemma~\ref{lemma_carrier}, the carrier–passenger separation of {\bf P1} in Lemma~\ref{lemma_passenger_caintaining}, the passenger–target navigation of {\bf P2} in Lemma~\ref{lemma_passenger_converging}, and the carrier–target avoidance of {\bf P3} in Lemma~\ref{lemma_overlapping_avoidance}. Finally, Theorem~\ref{theorem_smooth_seperation} summarizes our main result, i.e., the fulfillment of  {\bf P1}–{\bf P3}. 

\begin{lemma}
\label{lemma_carrier}
Under Assumptions~\ref{initial_carrier_target_position}-\ref{initial_carrier_passenger_position}, the carrier robot governed by \eqref{controller_carrier}  asymptotically approaches the fixed and hazardous target before the carrier-passenger separation, in which case the carrier robot stops, i.e., there exists a finite time $T>0$ such that (i) $\|\bold{e}_\textrm{t,c}(t)\|$ monotonically decreases for all $t\in[0,T]$, (ii) $\bold{e}_\textrm{t,c}(t)$ is constant for all $t>T$ with 
 $\bold{e}_\textrm{t,c}$ be given by \eqref{err_passenger_carrier_carrier}.
\end{lemma}

\begin{proof}
We will prove the existence of the time $T$, and show that the scenarios (i)-(ii) in Lemma~\ref{lemma_carrier} hold for such a time $T$.
 

(i) For the scenario of $\|\bold{e}_\textrm{t,c}(t)\|$ monotonically decreasing for all $t\in[0,T]$, since $\bold{e}_\textrm{t,c}$ in \eqref{err_passenger_carrier_carrier} is interconnected to $\bold{e}_\textrm{p,c}$ by \eqref{controller_carrier}, we first analyze the relationship between the time-dependent $\bold{e}_\textrm{p,c}(t)$ and $\bold{e}_\textrm{t,c}(t)$.
For $\bold{e}_\textrm{p,c}=0$, it follows from \eqref{Poteltion_field} that $P(\bold{e}_\textrm{p,c}, \bold{e}_\textrm{t,c})=-k_\textrm{p}d\|\bold{e}_\textrm{t,c}\|(-{\|\bold{e}_\textrm{t,c}\|}/{b}+c)$ becomes a parabolic function of $\|\bold{e}_\textrm{t,c}\|$ with the following properties: $P(\bold{e}_\textrm{p,c}, \bold{e}_\textrm{t,c})\geq0$ and $P(\bold{e}_\textrm{p,c},\bold{e}_\textrm{t,c})$ decreases whenever $\|\bold{e}_\textrm{t,c}\|$ decreases and 
$\|\bold{e}_\textrm{t,c}\|\geq bc$. 
This implies that when $\bold{e}_\textrm{p,c}=0, \|\bold{e}_\textrm{t,c}\|\geq bc$, the passenger's input $\bold{u}_\textrm{p}$ in \eqref{controller_passengerr} becomes $\bold{u}_\textrm{p}=-k_\textrm{p}P(\bold{e}_\textrm{p,c}, \bold{e}_\textrm{t,c})\bold{e}_\textrm{p,c}+\bold{u}_\textrm{c}=\bold{u}_\textrm{c}$, which further implies that the carrier-passenger relative distance remains zero, i.e., 
\begin{align}
\label{carrier_relation_1}
\bold{e}_\textrm{p,c}(0)=0 \land \|\bold{e}_\textrm{t,c}(t)\|\geq bc \Rightarrow \bold{e}_\textrm{p,c}(t)=0.
\end{align}
On the other hand, when $\bold{e}_\textrm{p,c}=0$, it follows from \eqref{controller_carrier} that the carrier's input becomes $\bold{u}_\textrm{c}=k_\textrm{c}\bold{e}_\textrm{t,c}$. Let us define now the following Lyapunov function candidate 
\begin{align}
\label{V_carrier}
V_1=\frac{1}{2}\bold{e}_\textrm{t,c}\t\bold{e}_\textrm{t,c}
\end{align}
with $\bold{e}_\textrm{t,c}$ be given by \eqref{err_passenger_carrier_carrier}. By taking the derivative of \eqref{V_carrier} and using \eqref{err_dynamic_carrier_target}, we can obtain that $\dot{V}_1(\bold{e}_{t,c})=-k_\textrm{c}\|\bold{e}_\textrm{t,c}\|^2\leq0$ whenever $\|\bold{e}_\textrm{p,c}(t)\|=0$.
This implies that $\|\bold{e}_\textrm{t,c}(t)\|$ monotonically decreases when~$\|\bold{e}_\textrm{p,c}(t)\|=0$ and  $\|\bold{e}_\textrm{t,c}(t)\|\geq bc$. Thus in combination with \eqref{carrier_relation_1}, we have the following implication 
\begin{align*}
& \bold{e}_\textrm{p,c}(0)=0, \|\bold{e}_\textrm{t,c}(t)\|\geq bc \\ 
& \Rightarrow \|\bold{e}_\textrm{t,c}(t)\| \, \text{is monotonically decreasing.} 
\end{align*}
Accordingly, from $\|\bold{e}_\textrm{t,c}(0)\|>\eta>bc$ and $\|\bold{e}_\textrm{p,c}(0)\|=0$ in Assumptions~\ref{initial_carrier_target_position} and \ref{initial_carrier_passenger_position}, one has that there always exists a finite time $t=T$ such that
$\|\bold{e}_\textrm{t,c}(t)\|\geq bc$ and $\|\bold{e}_\textrm{t,c}(t)\|$ is monotonically decreasing for all $t\in[0, T]$. Then, the proof of the existence of $T$, and the validity of (i) are both completed.

(ii) We will now analyze the relationship between $\bold{e}_\textrm{t,c}(t)$ and $\bold{e}_\textrm{p,c}(t)$ in the scenario (ii) of 
$\bold{e}_\textrm{t,c}(t)$ being constant for all $t>T$. In this case, when $t=T^+$ with $T^+$ be the right limit of $T$, we have $\|\bold{e}_\textrm{t,c}(T^+)\|<bc$ and $\|\bold{e}_\textrm{p,c}(T^+)\|-\|\bold{e}_\textrm{t,c}(T^+)\|<0$, which implies that $P(\bold{e}_\textrm{p,c}(T^+), \bold{e}_\textrm{t,c}(T^+))<0$ in \eqref{Poteltion_field}. Accordingly, the relative input between carrier and passenger robots at $t=T^+$ satisfies $\bold{u}_\textrm{p}(T^+)-\bold{u}_\textrm{c}(T^+)=P(\bold{e}_\textrm{p,c}(T^+), \bold{e}_\textrm{t,c}(T^+))\bold{e}_\textrm{p,t}$, which shows that the passenger robot separates from the carrier robot and moves toward to the target. In other words, we have the following implication 
\begin{align}
\label{lemma1_condition_1_5}
\nonumber & \bold{e}_\textrm{t,c}(T^+)<bc \Rightarrow\|\bold{e}_\textrm{p,c}(T^{+})\|\neq0 \,\,\land \\ & \|\bold{e}_\textrm{p,c}(t)\| \, \text{is monotically increasing } 
\forall t\in(T, \infty).
\end{align}
Since $\|\bold{e}_\textrm{p,c}(t)\|\neq0, \forall t>T$ in \eqref{lemma1_condition_1_5}, it follows from \eqref{controller_carrier} that $\bold{u}_\textrm{c}(t)=\bold{0}_\textrm{n}, \forall t>T$,
which implies that 
\begin{align}
\label{lemma1_condition_2}
& \|\bold{e}_\textrm{t,c}(t)\|<bc \land \|\bold{e}_\textrm{t,c}(t)\|~\mathrm{constant} \, \,\forall t>T.
\end{align}
It is worth noting that the time derivative of $V_1$ is discontinuous at the switching surface when the state $\|\bold{e}_\textrm{p,c}\| = 0$ transitions to $\|\bold{e}_\textrm{p,c}\|\neq 0$ in \eqref{err_dynamic_carrier_target}, which indicates that $V_1$ is not continuously differentiable at $t = T$. We compute the right upper Dini derivative of $V_1$ in \eqref{V_carrier} at $t = T$,
\begin{align*}
D^+V_1(T)=\lim\sup_{h\rightarrow0^+}\frac{V_1(T+h)-V_1(T)}{h}=\lim\sup_{h\rightarrow0^+}\frac{0}{h}=0,
\end{align*}
because of $V_1(T+h)=V_1(T)$. It further implies that $V_1(T)$ remains fixed and the error system of $\bold{e}_\textrm{t,c}(t)$ is stable at $t=T$. Therefore, the carrier robot stops once the passenger robot is separated, i.e., the validity of the scenario (ii) is completed. This concludes the proof. 
\end{proof}


\begin{lemma}
\label{lemma_passenger_caintaining}
Under Assumptions~\ref{initial_carrier_target_position}-\ref{initial_carrier_passenger_position}, the passenger robot governed by~\eqref{controller_passengerr} remains on the carrier robot and achieves the carrier-passenger separation in a finite time, 
i.e., there exists $T>0$ such that $\|\bold{e}_\textrm{p,c}(t)\|=0$ for all $t\in[0, T]$ and $\|\bold{e}_\textrm{p,c}(T^{+})\|\neq0$. 
\end{lemma}

\begin{proof}

From $P(\bold{e}_\textrm{p,c}(0), \bold{e}_\textrm{t,c}(0))>0$ in Lemma~\ref{lemma_carrier},
we analyze the carrier-passenger relative input $\bold{u}_\textrm{p}-\bold{u}_\textrm{c}$ in \eqref{controller_carrier} and \eqref{controller_passengerr} under $P(\bold{e}_\textrm{p,c}, \bold{e}_\textrm{t,c})\geq0$ as follows,
\begin{align}
\label{controller_passengerr_1}
\bold{u}_\textrm{p}-\bold{u}_\textrm{c}=&
-k_\textrm{p}(\|\bold{e}_\textrm{p,c}\|-\|\bold{e}_\textrm{t,c}\|)(\|\bold{e}_\textrm{p,c}\|+d)\left(\|\bold{e}_\textrm{p,c}\|\nonumber \bluff\right.\\
&\left.-\frac{\|\bold{e}_\textrm{t,c}\|}{b}+c\right)\bold{e}_\textrm{p,c}. 
\end{align}
By taking the derivative of the carrier-passenger error $\bold{e}_\textrm{p,c}$ in \eqref{marsupial_dynamics} and using \eqref{controller_passengerr_1}, we obtain 
\begin{align}
\label{derivative_pm_1}
\dot{\bold{e}}_\textrm{p,c}=&-k_\textrm{p}(\|\bold{e}_\textrm{p,c}\|-\|\bold{e}_\textrm{t,c}\|)(\|\bold{e}_\textrm{p,c}\|+d)\left(\|\bold{e}_\textrm{p,c}\|\nonumber\bluff\right.\\
&\left.-\frac{\|\bold{e}_\textrm{t,c}\|}{b}+c\right)\bold{e}_\textrm{p,c}.
\end{align}
Let us consider the following Lyapunov function candidate 
\begin{align}
\label{V2}
V_2=\frac{1}{2}\bold{e}_\textrm{p,c}\t\bold{e}_\textrm{p,c}.
\end{align}
Using \eqref{derivative_pm_1}, the time-derivative of $V_2$ satisfies
\begin{align}
\label{d_V2}
\dot{V}_2=&-k_\textrm{p}(\|\bold{e}_\textrm{p,c}\|-\|\bold{e}_\textrm{t,c}\|)(\|\bold{e}_\textrm{p,c}\|+d)\left(\|\bold{e}_\textrm{p,c}\|\nonumber\bluff\right.\\
&\left.-\frac{\|\bold{e}_\textrm{t,c}\|}{b}+c\right)\|\bold{e}_\textrm{p,c}\|^2,
\end{align}
which, together with $\|\bold{e}_\textrm{p,c}\|\geq0$,  gives that the invariance set of $\{\bold{e}_\textrm{p,c}\in\mathbb{R}^n~\big |~\dot{V}_2(\bold{e}_\textrm{p,c})=0\}$ in \eqref{d_V2} contains two manifolds and one point, i.e., 
\begin{align}
\label{V_3_equilibrium}
\{&\|\bold{e}_\textrm{p,c}\|=\|\bold{e}_\textrm{t,c}\|, \|\bold{e}_\textrm{p,c}\|=\frac{\|\bold{e}_\textrm{t,c}\|}{b}-c, \nonumber\\
&\bold{e}_\textrm{p,c}=\bold{0}_\textrm{n}~\big|~\dot{V}_2(\bold{e}_\textrm{p,c})=0 \}.
\end{align}
Recalling $\|\bold{e}_\textrm{p,c}(0)\|=0$ and $\|\bold{e}_\textrm{t,c}(0)\|>bc$ 
in Assumptions~\ref{initial_carrier_target_position}-\ref{initial_carrier_passenger_position}, one has that 
$V_2(0)=\bold{e}_\textrm{p,c}(0)\t\bold{e}_\textrm{p,c}(0)/2=0$ in \eqref{V2}. Hence the derivative $\dot{V}_2(0)=0$ in \eqref{d_V2} only contains one feasible point, i.e., $\{\bold{e}_\textrm{p,c}=\bold{0}_\textrm{n}~\big|~ \dot{V}_2(\bold{e}_\textrm{p,c}(0))=0\}$ and the non-positivity condition of $\dot{V}_2$ at $t = 0$ is guaranteed by the conditions of $\|\bold{e}_\textrm{p,c}(0)\|-\|\bold{e}_\textrm{t,c}(0)\|<0$ and $\|\bold{e}_\textrm{p,c}(0)\|-{\|\bold{e}_\textrm{t,c}(0)\|}/{b}+c<0$.
Combining with the fact that $\|\bold{e}_\textrm{t,c}(t)\|\geq bc$ for all $t\in[0, T]$ as shown before in the case~(i) in the proof of Lemma~\ref{lemma_carrier}, one has that $\dot{V}_2(t)\leq 0$ for all $t\in[0, T]$ when $\|\bold{e}_\textrm{p,c}(0)\|=0$; thus satisfying the non-positivity condition for $t\in[0, T]$. By using the LaSalle’s invariance principle~\cite{khalil2002nonlinear}, one has that $\|\bold{e}_\textrm{p,c}\|$ governed by \eqref{derivative_pm_1} satisfies
\begin{align}
\label{lemma2_condition}
\|\bold{e}_\textrm{p,c}(t)\|=0,~\forall t\in[0, T].
\end{align} 
Note that for all $t\in[0, T]$, the evolution of $\|\bold{e}_\textrm{t,c}(t)\|\geq bc$ 
in scenario (i) in the proof of Lemma~\ref{lemma_carrier} does not influence the evolution of $\bold{e}_\textrm{p,c}(t)$ 
because $\bold{e}_\textrm{p,c}$ already maintains at the  
equilibrium $\bold{e}_\textrm{p,c}(t)=0$. 
Meanwhile, at the switching point with $P(\bold{e}_\textrm{p,c}, \bold{e}_\textrm{t,c})=0$, it follows from \eqref{controller_passengerr} and \eqref{V2} that the left and right upper Dini derivatives become
$D^-V_2=\lim\sup_{P\rightarrow 0^{-}} P(\bold{e}_\textrm{p,c}, \bold{e}_\textrm{t,c})\bold{e}_\textrm{p,c}\t \bold{e}_\textrm{p,t}=0$ and $D^+V_2=\lim\sup_{P\rightarrow 0^{+}}-P(\bold{e}_\textrm{p,c}, \bold{e}_\textrm{t,c}\bold{e}_\textrm{p,c}\t \bold{e}_\textrm{p,c}=0$, respectively. 
It implies that $V_2$ is smooth at $P(\bold{e}_\textrm{p,c}, \bold{e}_\textrm{t,c})=0$ and the error system $\bold{e}_\textrm{p,c}$ is stable. Correspondingly, it follows from \eqref{lemma1_condition_1_5} that $\|\bold{e}_\textrm{p,c}(T^{+})\|\neq0$ is satisfied. 
\end{proof}


\begin{lemma}
\label{lemma_passenger_converging}
Under Lemma~\ref{lemma_passenger_caintaining}, the passenger robot governed by~\eqref{controller_passengerr}  
asymptotically converges to the fixed and hazardous target point after the carrier-passenger separation, i.e., $\lim_{t\rightarrow\infty}\bold{x}_\textrm{p}(t)=\bold{x}_\textrm{t}$. 
\end{lemma}

\begin{proof} 
Recalling $\|\bold{e}_\textrm{p,c}(t)\|\neq 0$ for all $t>T$ and $P(\bold{e}_\textrm{p,c}(T^+)$, $\bold{e}_\textrm{t,c}(T^+))<0$ in Lemma~\ref{lemma_carrier}, one has that the control inputs $\bold{u}_\textrm{c}, \bold{u}_\textrm{p}$ in \eqref{controller_carrier}, \eqref{controller_passengerr} of the carrier and passenger robots satisfy
\begin{align}
\label{controller_carrier_passenger_lemme_convergence}
\bold{u}_\textrm{c}=\bold{0}_\textrm{n},~\bold{u}_\textrm{p}=&
k_\textrm{p}(\|\bold{e}_\textrm{p,c}\|-\|\bold{e}_\textrm{t,c}\|)(\|\bold{e}_\textrm{p,c}\|+d)\left(\|\bold{e}_\textrm{p,c}\|\nonumber\bluff\right.\\
&\left.-\frac{\|\bold{e}_\textrm{t,c}\|}{b}+c\right)\bold{e}_\textrm{p,t}+\bold{u}_\textrm{c}.
\end{align}
Since for all $t>T$, $\|\bold{e}_\textrm{t,c}(t)\|<bc$  is constant 
in \eqref{lemma1_condition_2}, one has that the evolution of $\|\bold{e}_\textrm{t,c}(t)\|$ does not influence the  evolution of $\|\bold{e}_\textrm{p,c}(t)\|$ for all $t\in(T,\infty)$. 
Let us define the following Lyapunov function candidate $V_3$ 
\begin{align}
\label{V3}
V_3=\frac{1}{2}(\bold{e}_\textrm{p,c}-\bold{e}_\textrm{t,c})\t (\bold{e}_\textrm{p,c}-\bold{e}_\textrm{t,c}).
\end{align}
Since $\bold{e}_\textrm{p,c}-\bold{e}_\textrm{t,c}=\bold{x}_\textrm{p}-\bold{x}_\textrm{c}-\bold{x}_\textrm{t}+\bold{x}_\textrm{c}=\bold{x}_\textrm{p}-\bold{x}_\textrm{t}=\bold{e}_\textrm{p,t}$, it follows from \eqref{marsupial_dynamics} and \eqref{controller_carrier_passenger_lemme_convergence} that the derivative of $V_3$ in \eqref{V3} is given by 
\begin{align}
\label{d_V3}
\dot{V}_3
               =&k_\textrm{p}(\|\bold{e}_\textrm{p,c}\|-\|\bold{e}_\textrm{t,c}\|)(\|\bold{e}_\textrm{p,c}\|+d)\left(\|\bold{e}_\textrm{p,c}\|\nonumber\bluff\right.\\
&\left.-\frac{\|\bold{e}_\textrm{t,c}\|}{b}+c\right)\|\bold{e}_\textrm{p,t}\|^2.
\end{align}
It follows from above that the invariance set of $\{\bold{e}_\textrm{p,c}\in\mathbb{R}^n~\big |~\dot{V}_3(\bold{e}_\textrm{p,c})=0\}$ in~\eqref{d_V3} also contains two manifolds and one point, i.e., 
\begin{align}
\label{V_3_equilibrium}
\{&\|\bold{e}_\textrm{p,c}\|=\|\bold{e}_\textrm{t,c}\|, \|\bold{e}_\textrm{p,c}\|={\|\bold{e}_\textrm{t,c}\|}/{b}-c, \nonumber\\
&\bold{e}_\textrm{p,c}=\bold{e}_\textrm{t,c}~\big|~ \dot{V}_3(\bold{e}_\textrm{p,c})=0 \},
\end{align}
with $\bold{e}_\textrm{p,t}:=\bold{e}_\textrm{p,c}-\bold{e}_\textrm{t,c}=\bold{0}_\textrm{n}$. At the beginning stage of the carrier-passenger separation (i.e., $t=T^+$), it follows from Lemma~\ref{lemma_passenger_caintaining} that $\|\bold{e}_\textrm{p,c}(T^+)\|\approx 0$, $\|\bold{e}_\textrm{t,c}(T^+)\|\approx bc$ and ${\|\bold{e}_\textrm{t,c}(T^+)\|}/{b}-c<0$, which implies that $\|\bold{e}_\textrm{p,t}(T^{+})\|>0$, $\|\bold{e}_\textrm{p,c}(T^+)\|-\|\bold{e}_\textrm{t,c}(T^+)\|<0$ and ${\|\bold{e}_\textrm{p,c}(T^+)\|-\|\bold{e}_\textrm{t,c}(T^+)\|}/{b}+c>0$.
It follows from \eqref{d_V3} that $\dot{V}_3(T^+)<0$, which by the definition of $V_3$ means that 
$\|\bold{e}_\textrm{p,t}\|$ decreases and $\|\bold{e}_\textrm{p,c}\|$ increases at $t=T^+$. 
In combination with $\bold{u}_\textrm{c}, \bold{u}_\textrm{p}$ in \eqref{controller_carrier_passenger_lemme_convergence}, one has that $\|\bold{e}_{p,t}\|$ decreases, $\|\bold{e}_\textrm{p,c}\|-\|\bold{e}_\textrm{t,c}\|<0$, and $\|\bold{e}_\textrm{p,c}\|-\|\bold{e}_\textrm{t,c}\|$ increases whenever 
$\|\bold{e}_\textrm{p,c}(T^+)\|-\|\bold{e}_\textrm{t,c}(T^+)\|<0$. Hence, the manifold of $\|\bold{e}_\textrm{p,c}\|=\|\bold{e}_\textrm{t,c}\|$ in \eqref{V_3_equilibrium} is only satisfied when $\bold{e}_\textrm{p,c}=\bold{e}_\textrm{t,c}$, and $\|\bold{e}_\textrm{p,c}(t)\|-\|\bold{e}_\textrm{t,c}(t)\|\leq0$ for all $t>T$ whenever 
$\|\bold{e}_\textrm{p,c}(T^+)\|-\|\bold{e}_\textrm{t,c}(T^+)\|<0$. Meanwhile, from the fact that
\begin{align}
\label{carrier_invariant_condition}
&(\|\bold{e}_\textrm{p,c}(t)\|+c)b>\|\bold{e}_\textrm{t,c}(t)\|>0,  \forall t>T,
\end{align}
one has that the invariant manifold $\|\bold{e}_\textrm{p,c}\|={\|\bold{e}_\textrm{t,c}\|}/{b}-c$ in \eqref{V_3_equilibrium} are excluded. 
Accordingly, the invariance set of $\{\bold{e}_\textrm{p,c}\in\mathbb{R}^n~\big |~\dot{V}_3(\bold{e}_\textrm{p,c}(t))=0, \forall t>T\}$ in~\eqref{d_V3} essentially contains one point $\{\bold{e}_\textrm{p,c}(t)=\bold{e}_\textrm{t,c}(t)~(\mbox{i.e.}, \bold{e}_\textrm{p,t}(t)=\bold{0}_\textrm{n}), \forall t>T\}$. Recalling $\|\bold{e}_\textrm{p,c}(t)\|-\|\bold{e}_\textrm{t,c}(t)\|\leq0$ for all $t>T$ and \eqref{carrier_invariant_condition}, one has that the non-positivity condition of $\dot{V}_3(t)\leq0$ for all $t>T$ is guaranteed as well. Analogous to the proof in Lemma \ref{lemma_passenger_caintaining}, one has that $\lim_{t\rightarrow\infty}\bold{e}_\textrm{p,t}(t)=\bold{0}_\textrm{n}$. This concludes the proof.
\end{proof}

\begin{lemma}
\label{lemma_overlapping_avoidance}
The carrier robot governed by \eqref{controller_carrier} also achieves avoidance with the fixed and hazardous target all along, i.e., $\|\bold{e}_\textrm{t,c}(t)\|\neq 0, \forall t>0$ in {\bf P3} of Definition~\ref{def_peparation_navigation}.  
\end{lemma}

\begin{proof}
It follows from the scenario (i) in Lemma~\ref{lemma_carrier} 
that $\|\bold{e}_\textrm{t,c}(t)\|\geq bc, \forall t\in[0, T].$
Combining it with $\|\bold{e}_\textrm{t,c}(t)\|>0,~\forall t>T$ in \eqref{lemma1_condition_2}, the proof is completed.
\end{proof}

\begin{theorem}
\label{theorem_smooth_seperation}
Under Assumptions~\ref{initial_carrier_target_position}-\ref{initial_carrier_passenger_position}, the marsupial robotic system $\mathcal V$ governed by \eqref{marsupial_dynamics}, \eqref{controller_carrier}, \eqref{controller_passengerr} 
achieves marsupial separation \& navigation to the fixed target point. 
\end{theorem}

The proof follows  
from the application of Lemmas \ref{lemma_carrier}-\ref{lemma_overlapping_avoidance}. 

\begin{figure}[!htb]
\centering
\includegraphics[width=6.0cm]{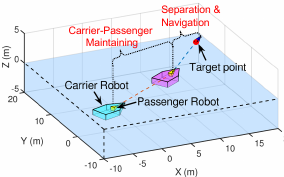}
\caption{Trajectories of carrier-passenger separation and navigation in 3D. Note that the carrier and passenger robots are modeled in 3D space. However, the motion of the carrier robot is constrained such that its $z$-coordinate remains zero. To address the issue, $\bold{u}_\textrm{c}$ in \eqref{controller_carrier} is first designed in 3D, and then projected it onto the 2D plane (i.e., by setting $u_{c,3} = 0$ in \eqref{marsupial_dynamics}) to reflect the planar behavior. Additionally, we assume that the target’s $z$-axis position satisfies $|x_{t,3}| \leq bc$, which guarantees $\|\bold{e}_\textrm{t,c}(T)\| = bc$ for some time $T$.} 
\label{3D_trajectory}
\end{figure}
\begin{figure}[!htb]
\centering
\includegraphics[width=7.0cm]{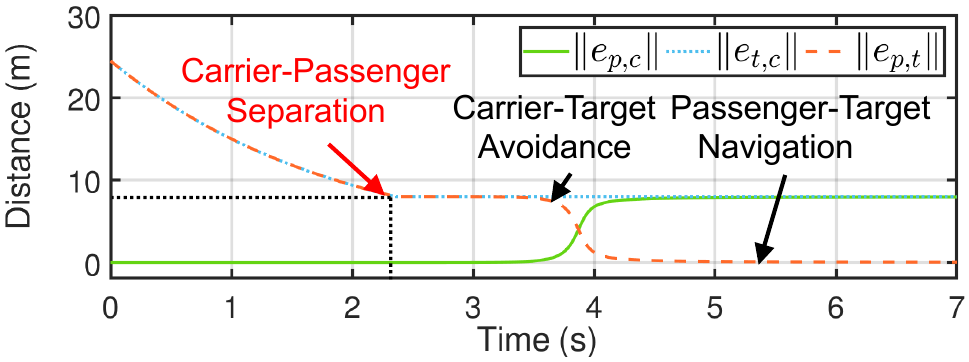}
\caption{Temporal evolution of the carrier-passenger, the carrier-target, and the passenger-target distances $\|\bold{e}_\textrm{p,c}\|$, $\|\bold{e}_\textrm{t,c}\|$, $\|\bold{e}_\textrm{p,t}\|$, respectively in Fig.~\ref{3D_trajectory}.  }
\label{3D_ptate}
\end{figure}

\section{Simulations}
\label{sec_pimulation}
In this section, the control gains for $\bold{u}_\textrm{c}, \bold{u}_\textrm{p}$ in \eqref{controller_carrier} and \eqref{controller_passengerr} are set to be $k_\textrm{c} = 0.5$ and $k_\textrm{p} = 1$, respectively. The parameters of $P(\bold{e}_\textrm{p,c}, \bold{e}_\textrm{t,c})$ in \eqref{Poteltion_field} are chosen to be $b = 8$, $c = 1$, and $d = 1$. 
\begin{figure}[!htb]
\centering
\includegraphics[width=6.2cm]{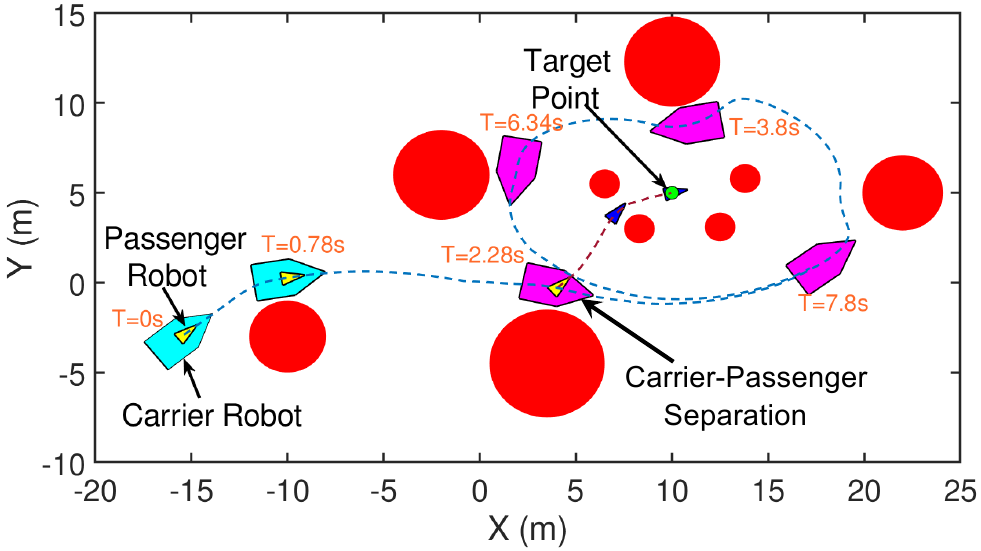}
\caption{Trajectories of the carrier-passenger marsupial robots from the initial positions (i.e., the blue vessel and small yellow triangle represent the carrier and passenger robots, respectively) to the carrier-passenger separation and navigation in obstacle environments (i.e., the magenta vessel and small dark blue triangle represent the carrier and passenger robots). All the symbols have a similar meaning in Fig.~\ref{3D_trajectory}. Note that the carrier does not
stop at the moment of the separation as discussed in Remark~\ref{remark_1}.}
\label{2D_trajectory}
\end{figure}
\begin{figure}[!htb]
\centering
\includegraphics[width=6.5cm]{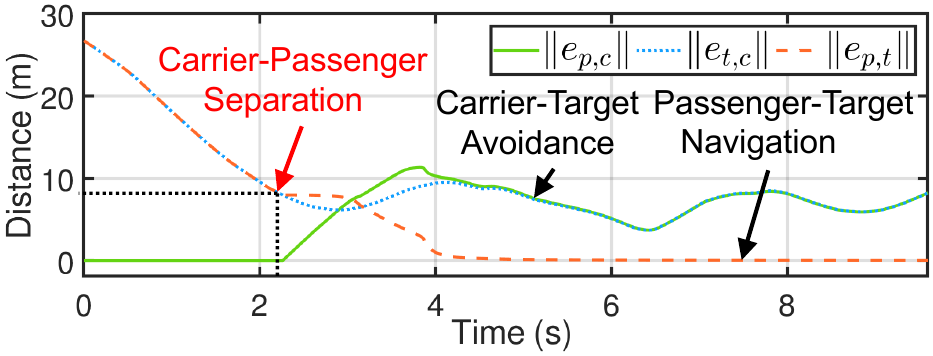}
\caption{Temporal evolution of the carrier-passenger, the carrier-target, and the passenger-target distances $\|\bold{e}_\textrm{p,c}\|$, $\|\bold{e}_\textrm{t,c}\|$, $\|\bold{e}_\textrm{p,t}\|$, respectively in Fig.~\ref{2D_trajectory}.  }
\label{2D_ptate}
\end{figure}
The initial positions of the carrier, passenger robots, and target point are set to satisfy Assumptions~\ref{initial_carrier_target_position}-\ref{initial_carrier_passenger_position}. 
Fig.~\ref{3D_trajectory} illustrates the trajectories from the initial positions to the final carrier-passenger separation and passenger-target navigation in 3D. 
Fig.~\ref{3D_ptate} presents the state evolution of $\|\bold{e}_\textrm{p,c}(t)\| = 0, \forall t \in [0, 2.4]$, and $\|\bold{e}_\textrm{p,c}(t)\| > 0, t > 2.4$, which confirms {\bf P1} of Definition~\ref{def_peparation_navigation}. 
As $\lim_{t \to \infty} \|\bold{e}_\textrm{p,t}\| = 0$, it indicates that {\bf P2} of Definition~\ref{def_peparation_navigation} is achieved. For the carrier-target distance, one has $\|\bold{e}_\textrm{t,c}(t)\|=8>0, \forall t > 2.4\ \text{s}$, verifying {\bf P3} of Definition~\ref{def_peparation_navigation}.
These results demonstrate the effectiveness of the proposed controllers \eqref{controller_carrier} and \eqref{controller_passengerr} in 3D. 
Additionally, for the reactive navigation in complex obstacle environments, the proposed controllers $\bold{u}_\textrm{c}, \bold{u}_\textrm{p}$ in \eqref{controller_carrier} and \eqref{controller_passengerr}, can conveniently be adapted by seamlessly combining them with quadratic programming control barrier functions (QP-CBF) \cite{ames2016control}. In Figs. \ref{2D_trajectory}-\ref{2D_ptate}, we show a simulation result where the carrier robot keeps moving and where our proposed controllers are integrated with QP-CBF method to avoid obstacles. The effectiveness of the proposed algorithm in obstacle environments is validated as well.




\section{Conclusion}
\label{sec_conclusion}
We have presented an equilibrium-driven controller for the marsupial robotic system to achieve smooth carrier-passenger separation and navigation. Future works are underway on combining the proposed methods with other multi-agent control laws to achieve complex marsupial robotic tasks. 

\appendices

\bibliographystyle{IEEEtran}
\bibliography{IEEEabrv,ref}

\end{document}